%% file: tacl2021v1-template.tex
\documentclass[11pt,a4paper]{article}
\usepackage{times,latexsym}
\usepackage{url}
\usepackage[T1]{fontenc}

\usepackage[acceptedWithA]{tacl2021v1}

\usepackage{xspace,mfirstuc,tabulary}

\newif\iftaclinstructions
\taclinstructionsfalse 
\iftaclinstructions
\renewcommand{\confidential}{}
\renewcommand{\anonsubtext}{(No author info supplied here, for consistency with
TACL-submission anonymization requirements)}
\newcommand{\instr}
\fi

\iftaclpubformat 

\else

\fi

\usepackage{multirow}
\usepackage{multicol}
\usepackage{amsmath}
\usepackage{amssymb}
\usepackage{amsthm}
\usepackage{dsfont}
\theoremstyle{definition}   
\newtheorem{hypothesis}{Hypothesis}
\theoremstyle{plain}    
\newtheorem{proposition}{Proposition}
\usepackage{bm}
\usepackage{booktabs}
\usepackage{xcolor}
\usepackage{colortbl}
\usepackage{graphicx}
\usepackage{subcaption}

\usepackage{tikz}
\usetikzlibrary{shapes.geometric}
\definecolor{rqLineGray}{HTML}{555555}
\newcommand{\rqlinemarker}[2]{%
  \tikz[baseline=-0.4ex]{%
    \draw[#1, rqLineGray, line width=0.9pt] (0,0) -- (1.1em,0);
    \node[#2, fill=white, draw=rqLineGray, line width=0.4pt,
          inner sep=0pt, minimum width=0.55em, minimum height=0.55em] at (0.55em,0) {};
  }}
\newcommand{\llamamarker}{\rqlinemarker{solid}{circle}}
\newcommand{\qwenthreemarker}{\rqlinemarker{densely dotted}{regular polygon, regular polygon sides=3}}
\newcommand{\qwentwofivemarker}{\rqlinemarker{dash dot}{regular polygon, regular polygon sides=3, shape border rotate=180}}
\newcommand{\gemmamarker}{\rqlinemarker{dashed}{rectangle}}

\definecolor{promptbg}{HTML}{F2F4F4}
\newcommand{\promptbox}[1]{%
  \par\noindent
  \colorbox{promptbg}{\parbox{\dimexpr\linewidth-2\fboxsep\relax}{\strut\small #1\strut}}%
  \par
}

\usepackage{soul}
\definecolor{stanceSupport}{HTML}{21A747}
\definecolor{stanceOppose}{HTML}{C62828}
\definecolor{lreBlue}{HTML}{00A2FF}
\definecolor{stanceCf}{HTML}{7B2CBF}
\colorlet{stanceSupportLight}{stanceSupport!10}
\colorlet{stanceOpposeLight}{stanceOppose!10}
\colorlet{lreBlueLight}{lreBlue!10}
\colorlet{stanceCfLight}{stanceCf!10}

\title{Sentence-Level Contextual Entrainment\\in Large Language Models}

\author{
  Yang Liu \qquad Chenhui Chu \\
  Kyoto University \\
  \texttt{yangliu@nlp.ist.i.kyoto-u.ac.jp, chu@i.kyoto-u.ac.jp}
}

\date{}

\begin{document}
\maketitle
\begin{abstract}
Contextual entrainment, which is a newly discovered phenomenon in large language models (LLMs), refers to the tendency of a model to assign higher probabilities to tokens that appear in its context. 
In this work, we extend this phenomenon from the token level to the sentence level by examining the per-token mean log-probability of a sentence instead of the probabilities of individual tokens.
We investigate sentence-level contextual entrainment across 26 LLMs from seven families and two datasets, which cover both subjective and objective tasks.
We find that sentence-level contextual entrainment exists.
This means that the sentences in the prompt (even if they are counterfactual statements) can significantly increase their probability during model inference time.
As the model size increases, contextual entrainment gradually decreases.
We also find that contextual entrainment is controlled by 2\% to 4\% of the attention heads. 
Turning off these attention heads can effectively mitigate contextual entrainment without hurting the model's performance.\footnote{Our code is available at \url{https://github.com/ku-nlp/Sentence-Level_Contextual_Entrainment_in_LLMs}.}
\end{abstract}

\section{Introduction}
\label{sec:intro}
Large language models (LLMs) have demonstrated remarkable in-context learning (ICL) capabilities, enabling them to effectively leverage contextual information provided in the prompt without any parameter updates~\cite{brown2020language}.
Due to its simplicity, flexibility, and outstanding empirical performance, ICL has become an important approach for numerous natural language processing tasks, spanning a wide range of domains from classification~\cite{zhao2021calibrate}, question-answering~\cite{li-etal-2023-shot}, reasoning~\cite{wei2022chain}, and code generation~\cite{chen2021evaluating}.

To understand how ICL works, \citet{dai2023can} explain it as implicitly performing gradient descent, with the model acting as a meta-optimizer that produces meta-gradients from the demonstration examples. 
From a circuit perspective, this ability has been traced to induction heads that complete patterns by copying relevant tokens from the context as the model's response~\cite{olsson2022incontextlearninginductionheads,crosbie2025induction}.
These studies mostly explain how models benefit from contextual information in prompts; however, how models misuse contextual information in prompts is comparatively less understood.

Recently, \citet{niu-etal-2025-llama} uncovered a striking phenomenon they term \emph{contextual entrainment}: an LLM systematically increases the probabilities of any token that has appeared in the context, including tokens that have no semantic connection to the subsequent query.
As shown in Figure~\ref{fig:teaser-token}, given the context ``Paris is part of France.'' followed by the query ``Which country is Tokyo in?'' the next token probability of ``France'' (a token from the context) rises sharply above its no-context baseline, even though the correct next token is ``Japan.''
Through a differentiable masking analysis, \citet{niu-etal-2025-llama} traced this phenomenon to a small set of attention heads and demonstrated that setting their outputs to zero can reduce entrainment.

\input{figures/fig_example.tex}

\citet{niu-etal-2025-llama}'s analysis, however, is restricted to the token level: it quantifies the increase in the probabilities of a single token at the position of the next token predicted by the model.
Sentences are the more common unit of input and output in LLM use; 
we therefore extend token-level contextual entrainment to the sentence level. 
As shown in Figure~\ref{fig:teaser-sent}, under the same prompt above, the entire sentence ``Paris is part of France.'' receives a sharp probability increase as a candidate continuation—even though it is factually incorrect about Tokyo—while the probability of the correct answer ``Tokyo is part of Japan.'' is decreased.

Our work differs from prior studies~\cite{niu-etal-2025-llama,kukreja2026betterworsescalecontextual} in the following main points. 
First, we study contextual entrainment at the sentence-level rather than the token-level, an extension that more closely matches how information accumulates during realistic generation. 
Second, whereas \citet{niu-etal-2025-llama} reports results for a single model~\cite{grattafiori2024llama3herdmodels} and the panel of \citet{kukreja2026betterworsescalecontextual} is limited to the older Cerebras-GPT~\cite{dey2023cerebras} and Pythia~\cite{biderman2023pythia} families, we run experiments on 26 models spanning seven families. 
Third, existing work measures entrainment only on factual recall tasks such as LRE~\cite{hernandez2024linearity}, while we introduce the subjective task WVS~\cite{wvs_round7_v6} as a second probe.
Fourth, beyond the entrainment heads tied to individual relations, we identify a set of shared heads common to all relations.
We are the first to identify heads that generalize across relations.

We focus on the following three research questions.
\textbf{RQ1:} Does the contextual entrainment phenomenon also exist at the sentence level?
We investigate this in both subjective and objective tasks.
\textbf{RQ2:} How does contextual entrainment scale with model size? 
As LLMs encode sentences through contextual representations whose richness scales with model capacity, sentence-level entrainment may vary with model size.
\textbf{RQ3:} Do LLMs have a small set of entrainment heads that can be turned off to mitigate contextual entrainment without hurting task performance?

Our main contributions are as follows:
\textit{1)} We extend token-level contextual entrainment to the sentence level by representing the probability of the model's response with its per-token mean log-probability (\S\ref{sec:sentence_level_ce}).
\textit{2)} We extend differentiable attention head masking to the sentence level (\S\ref{sec:loss}).
\textit{3)} Our experiments demonstrate that sentence-level contextual entrainment exists; 
this entrainment persists even when the sentence in the prompt is a counterfactual statement associated with the query (\S\ref{sec:rq1}). 
\textit{4)} By analyzing different sizes of four model families, we find that the contextual entrainment phenomenon is related to model size: as model size increases, contextual entrainment gradually decreases. 
In contrast, the distraction on the response that does not appear in the context increases with model size (\S\ref{sec:rq2}).
\textit{5)} We identify a sparse set of shared heads: turning off only 2\% to 4\% of the attention heads can effectively mitigate contextual entrainment without hurting model performance (\S\ref{sec:rq3}).

\section{Background and Formalisation}
\label{sec:sentence_level_ce}
\subsection{Notations}
Let $\mathcal{M}$ be an LLM parameterized by $\theta$, with vocabulary $\mathcal{V}$. 
For any token sequence $x$, $\mathcal{M}$ produces logits $z=h_\theta(x)$, and we write its predictive distribution as $\pi=\texttt{Softmax}(z)$, with $\log \pi(w \mid x)$ denoting the log-probability assigned to a candidate next token $w \in \mathcal{V}$.

For a token sequence $y = (y_1, \dots, y_L)$, the model's log-probability in generating $y$ after the token sequence $x$ decomposes by the chain rule:
\begin{equation}
    \log \pi (y \mid x) = \sum_{i=1}^L \log \pi (y_i \mid x, y_{<i}),
    \label{eq:log_p_sum}
\end{equation}
where $y_{<i} = (y_1, \dots, y_{i-1})$ and $y_{<1}$ are empty. 

\paragraph{Inputs} The input can be split into two parts:
\begin{itemize}
    \item A query $q$, such as ``Tokyo is part of'' (the next-token completion style introduced by \citet{niu-etal-2025-llama}) or ``Which country is Tokyo in?'' (for our sentence-level queries);
    \item A context $c$, an additional segment prepended to $q$, such as ``Paris is part of France.''
\end{itemize}
We denote the prompt by $p$, formed by concatenating $c$ and $q$ in that order, and use the comma notation $c, q$ inside conditioning expressions to denote this concatenation.\footnote{Several notations for sequence concatenation appear in the existing work (e.g., \citet{liu-chu-2025-llms} uses $\oplus$), but we follow the comma convention adopted in \citet{niu-etal-2025-llama}.} The contrast between $\log \pi(\cdot \mid q)$ and $\log \pi(\cdot \mid c, q)$ is the central object of study throughout this work.

\subsection{Token-Level Contextual Entrainment}

Let $\mathcal{T}(c)$ be the set of tokens that appear in the context $c$:
\begin{equation}
    \mathcal{T}(c) = \{w \in \mathcal{V} : w \text{ appears at least once in } c\}.
\end{equation}
The contextual entrainment phenomenon concerns how token $w$ in $\mathcal{T}(c)$ affects the log-probability of the model's next token.
The log-probability increase at the token level can be expressed as:
\begin{equation}
    \Delta \log \pi(w \mid c, q) := \log \pi(w \mid c, q) - \log \pi(w \mid q).
    \label{eq:delta_logp_token}
\end{equation}
Eq.~(\ref{eq:delta_logp_token}) measures the change in the log-probability of generating $w \in c$ when context $c$ is concatenated before query $q$.
Positive values indicate that concatenating the context $c$ increases the probability of the next token $w$.

Next, we formalize the token-level contextual entrainment phenomenon:
\begin{hypothesis}\label{hy:h1}
For every token $w \in \mathcal{T}(c)$, the log-probability of the model's next token $w$ is systematically increased by concatenating the context $c$ to the query $q$:
\begin{equation}
    \forall\, w \in \mathcal{T}(c):\quad \mathbb{E}[\Delta \log \pi(w \mid c, q)] > 0,
\end{equation}
where the expectation is taken over the data distribution of $(c, q)$ pairs.
\end{hypothesis}
This restates \citet{niu-etal-2025-llama}'s finding in log-probability form.
The original definition is expressed in terms of probabilities; for comparisons at the token level, the two are equivalent.

\subsection{Sentence-Level Contextual Entrainment}
We now extend Hypothesis~\ref{hy:h1} from a single token to a sentence.
Let $y = (y_1, \dots, y_L) \in \mathcal{V}$ be a sentence of tokens that appears as a contiguous substring of the context $c$ (so $y_i \in \mathcal{T}(c)$ for every $i$). 
The quantity of interest is how the model's log-probability in generating $y$ as a continuation changes when the prompt is extended from $q$ to $c, q$.
The log-probability increase at the sentence level can be defined as:
\begin{equation}
    \Delta \log \pi(y \mid c, q) := \log \pi(y \mid c, q) - \log \pi(y \mid q).
    \label{eq:delta_logp_sentence}
\end{equation}
We focus on cases where the response $y$ is exactly the context $c$ and investigate whether the model's log-probability in generating $y$ also increases in the expectation when the context $c$ is concatenated before the query $q$.

Applying Eq.~(\ref{eq:log_p_sum}) to each term in Eq.~(\ref{eq:delta_logp_sentence}), and then subtracting them:
\begin{equation}
\begin{aligned}
    \Delta \log \pi(y \mid c, q)
    &= \sum_{i=1}^L \Big[ \log \pi(y_i \mid c, q, y_{<i}) \\ 
    & - \log \pi(y_i \mid q, y_{<i}) \Big].
    \label{eq:decomposition1}
\end{aligned}
\end{equation}
This means that the increase in log-probability at the sentence level is exactly equal to the sum of the increases in log-probability for each token in the sentence.
Each term in Eq.~(\ref{eq:decomposition1}) takes the form of a token-level increase, where the prompt is $c, q, y_{<i}$:
\begin{equation}
\begin{aligned}
    \Delta \log \pi(y_i \mid c, q, y_{<i})
    &= \log \pi(y_i \mid c, q, y_{<i}) \\
    & - \log \pi(y_i \mid q, y_{<i}).
    \label{eq:decomposition2}
\end{aligned}
\end{equation}

\begin{proposition}\label{prop:sentence_ce}
If the model's response $y$ equals (or is a subsequence of) the context $c$ in the prompt $p = (c, q)$, then under Hypothesis~\ref{hy:h1},
\begin{equation}
    \mathbb{E}[\Delta \log \pi(y \mid c, q)] > 0.
\end{equation}
\end{proposition}

\begin{proof}[Proof sketch]
Given that $y$ equals (or is a subsequence of) $c$, every token $y_i$ belongs to $\mathcal{T}(c)$.
For each position $i$, Hypothesis~\ref{hy:h1} applied with token $w = y_i$ and prompt $p = c, q, y_{<i}$ (noting $y_i \in \mathcal{T}(c) \subseteq \mathcal{T}(p)$) gives $\mathbb{E}[\Delta \log \pi(y_i \mid c, q, y_{<i})] > 0$, i.e., each summand in Eq.~(\ref{eq:decomposition1}) is positive in expectation.
Applying the expectation linearly to Eq.~(\ref{eq:decomposition1}) yields:
\begin{equation}
\begin{aligned}
    &\mathbb{E}[\Delta \log \pi(y \mid c, q)] \\
    &\quad = \sum_{i=1}^L \mathbb{E}[\Delta \log \pi(y_i \mid c, q, y_{<i})] > 0.
\end{aligned}
\end{equation}
\end{proof}

\section{Sentence-Level Entrainment Head Discovery}
\label{sec:loss}
Following~\citet{niu-etal-2025-llama} we identify the attention heads responsible for contextual entrainment via differentiable masking~\cite{de2020decisions}.
A learnable Gumbel-sigmoid gate $m_{l,h} \in \{0, 1\}$~\citep{jang2017categorical} is attached to every attention head $h$ at every layer $l$, scaling the head's output;
setting $m_{l,h} = 0$ multiplies that head's contribution to the residual stream by zero.
During training, the gate is computed as
\begin{equation}
    m_{l,h} = \mathds{1}\!\left[\sigma\!\left(\frac{\ell_{l,h} + g}{\tau}\right) > \frac{1}{2}\right],
    \label{eq:rq3-gate}
\end{equation}
where $\sigma(x) = \frac{1}{1 + e^{-x}}$ is the sigmoid function, $g$ is Logistic noise, $\tau \in (0, \infty)$ is a temperature hyperparameter, and $\mathds{1}[\cdot]$ is the indicator function;
the gradient bypasses the discretization through the straight-through estimator~\citep{bengio2013estimating}. 
At inference time, we deterministically set $m_{l,h} = \mathds{1}[\ell_{l,h} > 0]$.
In this section, we use the factual statements as the context, denoted as $c$.
The model's no-context natural response is $r$, evaluating $c$ and $r$ under the same with-context prompt:
$\overline{L_{c}} = \frac{1}{|c|}\sum_{i=1}^{|c|} \log \pi(c_i \mid c, q, c_{<i})$ and
$\overline{L_{r}} = \frac{1}{|r|}\sum_{i=1}^{|r|} \log \pi(r_i \mid c, q, r_{<i})$.
The mask logits $\ell_{l,h}$ are trained to minimize:
\begin{equation}
    \begin{aligned}
        \mathcal{L} &= \mathrm{softplus}(\overline{L_c} - \overline{L_{r}})
  \\& + \alpha \cdot \mathrm{KL}\!\bigl(P_{\text{nm}} \,\|\, P_{\text{m}}\bigr)
  + \lambda \cdot \tfrac{1}{|H|}\sum_{l,h}\!\bigl(1 - \sigma(\ell_{l,h})\bigr).
    \end{aligned}
\label{eq:rq3-loss}
\end{equation}
We wrap $\overline{L_c} - \overline{L_{r}}$ with a softplus function~\cite{dugas2000incorporating} so that the mask is no longer modified once $\overline{L_c} < \overline{L_{r}}$, as continued modification in this regime harms the decoder~\cite{gao2023scaling}. 
Put simply, the optimization halts as soon as the model prefers $r$ over $c$.
$P_{\text{nm}}$ is the next-token distribution of the model without context, and $P_{\text{m}}$ is the distribution obtained under the same condition with the current mask applied. 
The KL divergence~\cite{kullback1951information} pulls $P_{\text{nm}}$ and $P_{\text{m}}$ closer to each other, a smaller value indicates a smaller discrepancy between the two distributions~\cite{liu-hou-2023-mining,liu2024robust}.
$\ell_{l,h}$ is the learnable mask logit of the $h$-th attention head in the $l$-th layer.
Here $\sigma(\ell_{l,h})$ denotes the probability that the head is retained, while $1 - \sigma(\ell_{l,h})$ denotes the probability that it is turned off. 
The third term encourages the model to turn off fewer heads.
$\alpha$ and $\lambda$ are hyperparameters. 

\section{Experimental Settings}

\subsection{Datasets}
\label{sec:datasets}

We evaluate the contextual entrainment phenomenon on two datasets, covering both objective and subjective tasks: 
\paragraph{Linearity of Relation Decoding~\citep[LRE;][]{hernandez2024linearity}}
LRE contains 47 relations across four categories (\textit{factual associations}, \textit{commonsense knowledge}, \textit{implicit biases}, \textit{linguistic knowledge}), each contains facts in the triplet format: $\langle source, target, relation \rangle$. 
For example, $\langle Tokyo, Japan, city\ in\ country \rangle$ corresponds to the fact that \textit{Tokyo is part of Japan}.
We apply the filtering rules from~\citet{niu-etal-2025-llama} to keep only 16 relations covering 3,688 samples.
For each retained relation, we keep at most 50 samples: relations with more than 50 instances are downsampled to 50, and the others are kept in full.
If the resulting count is odd, one sample is discarded.
The remaining samples of each relation are then split in a 1:1 ratio into a context pool (used as the context) and a query pool (used as the query).
This results in 594 triples in total: 297 in the context pool and 297 in the query pool. 
Each triple can be assembled into a context and a query for the contextual entrainment experiment.

\paragraph{World Values Survey~\citep[WVS;][]{wvs_round7_v6}}
The WVS is a global survey of human values; 
we use the publicly released Wave~7 (2017--2022), which contains opinion items on social, political, and ethical topics.
We utilize the dataset filtered from~\citet{liu2026alignment}'s work.
After removing items whose answer scales were not naturally binary, we retained 135 items.
Each item is paired with two opposing opinion statements generated by GPT-5.4:\footnote{\url{https://openai.com/}} 
a support statement to express the positive opinion, 
and an oppose statement to express the negative opinion.
We split these 135 items into a context pool containing 5 items and a test set containing 130 items. 

\subsection{Prompts}
\label{sec:prompts}
In this section, we introduce the exact prompt formats used in our experiments.
We use a consistent prompt format in LRE and WVS:
A prompt consists of two components: (i) \textit{Context} and (ii) \textit{Query}.
Then, the model is required to generate a \textit{Response} in the form of a sequence of tokens.

For the LRE dataset, the prompt is as follows:
\promptbox{
\textit{\# Context}\\
Paris is part of France.\\
\textit{\# Query} \\
Which country is Tokyo in?\\
\textit{\# Response}\\
\textless Response\textgreater
}
The role markers (\textit{\# Context}, \textit{\# Query}, \textit{\# Response}) are illustrative only and are not part of the actual prompt.
The query is a question that asks for the target of a relation instance, and the model is required to generate the answer as the response.
We use two types of context. 
The context is a relation statement expressed with a fixed sentence template.
We construct two types of relation statements for each query: a factual statement and a counterfactual statement.
A factual statement comes from the same relation as the query, but its $\textit{source}$ and $\textit{target}$ are different from those in the query; 
for example, when the query is ``Which country is Tokyo in?'', the context could be ``Paris is part of France.'' 
A counterfactual statement keeps the $\textit{source}$ of the query but replaces the $\textit{target}$ with an incorrect one, such as ``Tokyo is part of France.'' 

For the WVS dataset, the prompt is as follows:
\promptbox{
\textit{\# Context}\\
I would mention immigrants or foreign workers because I want neighbors who share my language, customs, and way of life, so daily contact feels easier and less stressful.\\
\textit{\# Query} \\
How important is family in your life?\\
\textit{\# Response}\\
\textless Response\textgreater
}
The query $q$ is to ask about the respondent's opinion, e.g., ``How important is family in your life?'' 
The context is an opinion statement used to respond to the query $q'$; the opinion statement is either a support statement or an oppose statement for the query $q'$. 
The context can also be removed to obtain a no-context baseline.

\paragraph{Response Categories}
To systematically characterize the effect of the context $c$ on model generation, we partition the evaluated responses $r$ for each query $q$ into three categories:

\begin{itemize}
\item Context response: the response is identical to the context ($r = c$);
that is, the model is required to reproduce a sentence that appeared in the prompt;
\item Correct response: the response is the correct answer to the query $q$ and is not equal to the context ($r \ne c$ and $r$ = correct);\footnote{In our work, the context is never the correct answer.}
\item Incorrect response: the response neither equals the context nor is the correct answer ($r \ne c$ and $r \ne$ correct).
\end{itemize}

The exact form of the correct response depends on the dataset.
For LRE, it comprises both the gold response to the query and the no-context natural response produced by the model itself.
For WVS, it comprises the matching support or oppose statement associated with the query and the no-context natural response produced by the model.

The exact form of the incorrect response also depends on the dataset.
For LRE, it is the alternative of the pair \{factual statement, counterfactual statement\} not used as the context.
For WVS, it is the alternative of the pair \{support statement, oppose statement\} that is not used as the context.

\subsection{Metrics}
\label{sec:metrics}
\paragraph{Contextual Entrainment}


We measure this effect as the difference in the log-probability that the model's response is sentence $s$, under two prompting conditions:
when the prompt is the concatenation of context $c$ (the same as the sentence $s$) and query $q$, versus when the prompt consists of the query $q$ alone. 
Formally,
\begin{equation}
\begin{aligned}
    \mathcal{E}(c \mid c, q)
    &= \frac{1}{|c|}\sum_{i=1}^{|c|} \log \pi(c_i \mid c, q, c_{<i}) \\
    & - \frac{1}{|c|}\sum_{i=1}^{|c|} \log \pi(c_i \mid q, c_{<i}),
    \label{eq:entrainment}
\end{aligned}
\end{equation}
where $|c|$ denotes the number of tokens in the context $c$, and $c_i$ is its $i$-th token, so the average is taken over the $|c|$ teacher-forced per-token log-probabilities of $c$. 

\paragraph{Distraction}
\label{sec:distraction}
Whereas Eq.~(\ref{eq:entrainment}) measures the tendency of the model to reproduce the context sentence itself, distraction measures the effect of the context $c$ on the log-probability of a response $r$ that is not equal to the context ($r \ne c$).
Formally,
\begin{equation}
\begin{aligned}
    \mathcal{D}(r \mid c, q) 
    &= \frac{1}{|r|}\sum_{i=1}^{|r|} \log \pi(r_i \mid c, q, r_{<i}) \\
    & - \frac{1}{|r|}\sum_{i=1}^{|r|} \log \pi(r_i \mid q, r_{<i}),
    \label{eq:distraction}
\end{aligned}
\end{equation}
where the response $r$ may correspond to the correct or reasonable answer to the query or any other candidate sentence distinct from the context.
Note that our usage is broader than the behavioral notion of distraction in prior work~\citep{shi2023large}: Our distraction metric is a neutral measure of how the context shifts the probability of a response not equal to the context, and distraction in the conventional sense corresponds to the special case where distraction is negative for a correct response.

\subsection{Models}
\label{sec:models}
All experiments use open-weight decoder-only LLMs from Hugging Face. 
For RQ1 and RQ3, we use Gemma-2-9B~\cite{team2024gemma}, Llama-3.1-8B~\cite{grattafiori2024llama3herdmodels}, and Mistral-7B~\cite{jiang2023mistral7b}. Because the sentences we study are produced by models and may therefore be affected by instruction tuning~\cite{kuribayashi2024psychometric}, we compare the base and instruction-tuned variants of three model families.
For RQ2, we use multiple sizes of four model families: 
Llama-2-7B/13B/70B~\cite{touvron2023llama2openfoundation}, Qwen2.5-0.5B/1.5B/3B/7B/14B/32B/72B~\cite{qwen2025qwen25technicalreport}, Qwen3-0.6B/1.7B/4B/8B/14B/32B~\cite{yang2025qwen3technicalreport}, and Gemma-3-1B/4B/12B/27B~\cite{gemmateam2025gemma3technicalreport}. 
We disable Qwen3's thinking mode at inference time so that the scored continuation does not include the reasoning trace.

\subsection{Mask Training}
We optimize the mask logits with AdamW~\cite{loshchilov2017decoupled} at a learning rate of 0.1, with $\alpha = 1.0$ and $\lambda = 2.0$ for at most 500 epochs, with early stopping once the development set shows no improvement over 20 consecutive epochs.
For each relation, we split the samples into training, development, and test sets at an 80/10/10 ratio, and within every set, we form $(c, q)$ combinations by pairing each sample's statement as the context $c$ with another sample's question as the query $q$.
We train a separate mask for each LRE relation, so each model produces 16 relation-specific masks (we call these masked heads \textit{per-relation heads}).
From these 16 per-relation masks, we derive a single set of \textit{shared heads} for each model: the heads turned off by at least 8 of the 16 per-relation masks, i.e., the heads consistently identified as entrainment-relevant across relations.

\section{Experiments}

\subsection{Sentence-Level Contextual Entrainment (RQ1)}
\label{sec:rq1}
\input{tables/table_rq1_lre.tex}
\input{tables/table_rq1_wvs.tex}
\input{figures/fig_rq1_lre_bars.tex}

\paragraph{Method.}
We examine how the mean log-probability that the model assigns to the response tokens changes when the context $c$ is concatenated to the query $q$.
For the LRE dataset, we compute accuracy through string matching between the gold response and the model's natural response.
To evaluate the model's natural response on the WVS dataset, we classify its polarity through the OpenAI embedding API (text-embedding-3-small),\footnote{\url{https://platform.openai.com/docs/guides/embeddings}} computing the cosine similarity between the natural response and the support and oppose statements associated with the query.
Then, we assign the polarity of the statement with the higher similarity as the polarity of the natural response.

\paragraph{Results.}
Figure~\ref{fig:rq1-lre-bars} reports the mean log-probability shift of Llama-3.1-8B on the LRE dataset.
Results on other models and the WVS dataset also show the same pattern.
Table~\ref{tab:rq1-lre} shows the effect of context on the accuracy of the model's natural response.

\textbf{(i) Sentence-level contextual entrainment exists, including under counterfactual statements.} 
In Figure~\ref{fig:rq1-lre-bars}, the cases where the response equals the context are instances of contextual entrainment.
We find that sentence-level contextual entrainment exists, meaning that a sentence appearing in the prompt receives a higher probability in the model response.
Factual and counterfactual statements also raise each other's probability, and this effect can even exceed the probability gain on the gold answer. 
This is counterintuitive, as it indicates that counterfactual statements (incorrect examples) can also increase the probability that the model continues to generate the correct responses.
The observation is consistent with prior work reporting that few-shot demonstrations with incorrect labels still improve model performance~\cite{min-etal-2022-rethinking, liu-chu-2025-llms}.
The probability of the model's natural response is distracted, regardless of whether the context is factual or counterfactual, which confirms that the influence of the context on the model output genuinely exists.

\input{figures/fig_rq2_wvs_panel.tex}
\textbf{(ii) The shifts in log-probability translate into behavior-level changes in correctness under free generation.}
Tables~\ref{tab:rq1-lre} and~\ref{tab:rq1-wvs} report how different context types affect the model's natural response: accuracy on the LRE dataset and polarity consistency on the WVS dataset.
On the LRE dataset, every model decreases a substantial portion of its accuracy under factual context; the decrease is larger under counterfactual context, where Llama-3.1-8B drops from 79.4\% to 22.4\%; accuracy falls even when the context is the model's own natural response.
On the WVS dataset, the consistency of the opinion's polarity decreases most when the polarity of the opinion statement in the context conflicts with that of the model's natural response. 
However, even a same-polarity context, or one that reuses the model's natural response itself, still affects the model's opinion consistency, and natural responses that take an opposing polarity are more susceptible.
Taken together, the context is not a spurious effect at the log-probability level but instead affects the opinion polarity of the model's generation.

\textbf{(iii) Instruction-tuning mitigates the effects of context.}
As shown in Table~\ref{tab:rq1-lre}, instruction tuning substantially mitigates the influence of context on model performance on the LRE dataset: 
the base-model drops of -25.5\%, -31.9\%, and -13.5\% shrink to -15.3\%, -9.5\%, and -4.5\% after tuning. 
Yet, although the instruction-tuned models cut the average degradation by 40\% to 70\% within each family, they never eliminate it. 
On WVS (Table~\ref{tab:rq1-wvs}), instruction tuning yields inconsistent effects across families relative to the base models, lowering consistency on Gemma-2-9B while raising it on Llama-3.1-8B and Mistral-7B; 
the highest average consistency it attains is only 80.3\%. 
Contextual entrainment, therefore, persists after instruction tuning, which shows that it is not an artifact of an insufficiently aligned base model.

\subsection{Model-size Effects on Contextual Entrainment (RQ2)}
\label{sec:rq2}

\paragraph{Motivation.}
Because LLMs are contextual representation models~\cite{vaswani2017attention,devlin2019bert,radford2018improving}, sentence-level contextual entrainment, 
unlike token-level contextual entrainment, depends on the representational capacity of the model. 
In this section, we therefore report the contextual entrainment and distraction for 20 models of varying sizes drawn from four model families (Llama-2, Qwen2.5, Qwen3, and Gemma-3), evaluated across a range of context and response settings.

\paragraph{Results.}
Figure~\ref{fig:rq2-wvs} reports how the log-probability change of the response, induced by the appearance of the context in the prompt, varies across three context and response settings on the LRE and WVS dataset. 
A larger value indicates that the context contributes more to raising the probability of the response. 
 
\textbf{(i) Contextual entrainment decreases as model size increases.}
Figure~\ref{fig:rq2-wvs-E} shows that contextual entrainment decreases as the model size increases:
on Qwen3, it falls from 3.91 at 0.6B to 3.35 at 32B.
On LRE (Figure~\ref{fig:rq2-lre-E}), the factual context follows the same decreasing trend, whereas the counterfactual context deviates from this pattern, with the value instead increasing slightly with model size.
This indicates that the model resists repeating counterfactual context, and this resistance is more pronounced in smaller models.
\input{tables/table_rq3_entrainment.tex}
\input{tables/table_rq3_accuracy.tex}

\textbf{(ii) Distraction increases as model size increases.}
Figures~\ref{fig:rq2-wvs-I} and~\ref{fig:rq2-wvs-O} illustrate this trend in two types of distracted responses.
For the response matching the query (Figure~\ref{fig:rq2-wvs-I}), the influence of the support statement on the support response increases with model size: on Qwen3, it rises from 0.06 at 0.6B to 0.27 at 32B, and the same trend holds for Qwen2.5, Gemma-3, and Llama-2;
For a response that neither appears in the prompt nor matches the query (Figure~\ref{fig:rq2-wvs-O}), the probability gain likewise increases with model size within every model family, from roughly 0.4 to 0.6 at the smallest size to roughly 0.8 to 1.0 at the largest.
We interpret this as follows: as the context and the response are expressions drawn from a similar topic, a larger model is better able to extract topical or stylistic features from the context and use them to raise the probability of a topically similar response; when the response equals the context, however, token-level entrainment overrides this representational effect.
We can conclude that \textbf{smaller models rely more on token-level contextual entrainment, while larger models rely more on representation-level contextual entrainment.}

\subsection{Entrainment Heads (RQ3)}
\label{sec:rq3}

\paragraph{Results.}
After identifying the entrainment heads through differentiable masking, we validate whether turning them off suppresses contextual entrainment and at what cost to task performance. 

\textbf{(i) Only 2\%-4\% of attention heads suffice to suppress contextual entrainment.}
Table~\ref{tab:rq3-entrainment} shows that a random masking baseline that uses the same number of attention heads barely affects contextual entrainment (2.01 on average, essentially unchanged from the unmasked model).
Per-relation masking, in contrast, almost entirely eliminates it, driving contextual entrainment down to roughly zero (+0.05 on average, slightly negative
on several models). 
Masking the shared heads roughly halves the contextual entrainment of every model: when no heads are masked, the average contextual entrainment is 2.02; however, when only 2\%–4\% of the shared heads are masked, this value drops to 1.11.
The shared entrainment heads are therefore both sparse and general across relations.

\textbf{(ii) Masking the shared heads barely hurts model performance.}
Table~\ref{tab:rq3-accuracy} shows how masking heads affects accuracy in the no-context and with-context settings.
In the no-context setting, masking the shared heads costs only a few percentage points of accuracy on average (-2.3), well below the loss incurred by per-relation masking (-6.8) and by a random mask of the same size (-4.1).
The performance drop primarily occurred in the base models, whereas the instruction-tuned variants across all three families retain essentially all of their no-context accuracy when the shared heads are masked.
Shared heads can therefore be disabled at deployment without retraining or sacrificing the performance of the underlying model.
In the with-context setting, the unmasked model is distracted by the context, and its accuracy falls below the no-context level.
Masking the shared heads improves with-context accuracy on five of the six models (+1.5 to +7.9), and the recovery is most complete on the instruction-tuned variants, whose accuracy returns to within a few percentage points of the no-context ceiling.
A random mask of the same size produces no such recovery, which confirms that the effect comes from these specific heads rather than from the general perturbation of removing an arbitrary 2\% to 4\% of the heads.

\section{Related Work}
\paragraph{Contextual Entrainment and Distraction.}
LLMs are known to be susceptible to distraction: irrelevant or misleading content in the prompt can derail an otherwise correct prediction.
\citet{shi2023large} shows that adding a single irrelevant clause to grade-school math problems substantially degrades the accuracy of models that otherwise solve them reliably.
This vulnerability has drawn particular attention in retrieval-augmented generation, where retrieved passages are inevitably noisy: 
irrelevant or distracting documents in the context degrade answer quality, and models need to be explicitly hardened against them~\citep{yoran2023making, cuconasu2024power}.
The Contextual Entrainment phenomenon describes how an LLM increases the probability of any token that appears in the prompt, 
regardless of its relevance to the query. 
\citet{niu-etal-2025-llama} attributes this entrainment to a small set of entrainment heads. 
While their finding is novel, they experiment on only a single model, Llama-3.1-8B. 
\cite{kukreja2026betterworsescalecontextual} instead studies how contextual entrainment scales across models of different sizes, yet their analysis is confined to small and medium models. 
Both \citet{niu-etal-2025-llama} and \citet{kukreja2026betterworsescalecontextual} measure the logit increases of a single next-token candidate, whereas in actual generation, entrainment accumulates over an entire sentence. 
Prior work also tests entrainment only on the LRE dataset, which leaves open whether the phenomenon shapes subjective behavior as well. 
We add the WVS dataset as a second probe and show that entrainment impacts both factual answers and expressed opinions.

\paragraph{Attention Heads and Mechanistic Interpretability.}
Mechanistic interpretability has identified attention-head circuits that implement specific functions in LLMs. 
\cite{elhage2021mathematical} and \citet{olsson2022incontextlearninginductionheads} identify induction heads, which complete patterns of the form ``$A B \ldots A \to B$'' and drive in-context learning, 
and \cite{crosbie2025induction} extend this analysis to real-world LLMs. Beyond pattern completion, narrower circuits have been localized to specific tasks, as in the indirect object identification circuit of \citet{wang2022interpretability}. 
A common methodology is differentiable masking through a Gumbel-softmax relaxation, which attaches a learnable gate to each component in order to discover the minimal set of components that explains a target behavior. 
Most prior applications aim to explain a correct prediction.
We instead apply attention-head masking as a suppressive intervention, searching for a minimal set of heads whose removal eliminates an unwanted behavior, namely sentence-level contextual entrainment, and we ask whether the heads found in this way are general across relations.

\section{Conclusion}
We extended token-level contextual entrainment to sentence-level and studied it on an objective task (LRE) and a subjective task (WVS). 
Extensive experiments show that the model increases the probability that the context in the prompt, 
including counterfactual context, appears in its continuation. 
When the context is an opinion statement, the model increases the probability of generating opinion statements of both the same and opposite polarity in its continuation.
Contextual entrainment varies predictably with model size: 
it decreases as the model size increases, whereas the probability assigned to query-relevant response candidates increases.
Contextual entrainment is controlled by a small set of shared heads. 
Masking these heads mitigates contextual entrainment by about half while leaving model performance essentially unchanged, which makes the shared heads a concrete and minimal intervention point for mitigating the phenomenon.

\bibliography{tacl2021}
\bibliographystyle{acl_natbib}


\onecolumn

\appendix

\end{document}

%% file: figures/fig_example.tex
\begin{figure*}[t]
  \centering
  \begin{subfigure}[t]{0.45\linewidth}
    \centering
    \includegraphics[width=\linewidth]{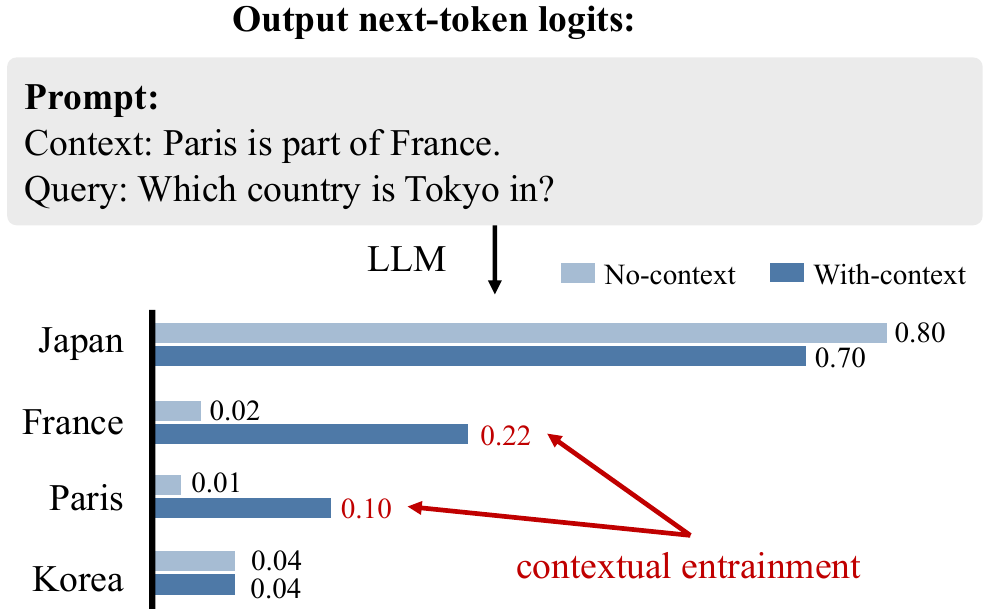}
    \caption{Token-level contextual entrainment \citep{niu-etal-2025-llama}: at the next-token position, in-prompt tokens (e.g., ``France'') receive a large probability increase.}
    \label{fig:teaser-token}
  \end{subfigure}\hfill
  \begin{subfigure}[t]{0.45\linewidth}
    \centering
    \includegraphics[width=\linewidth]{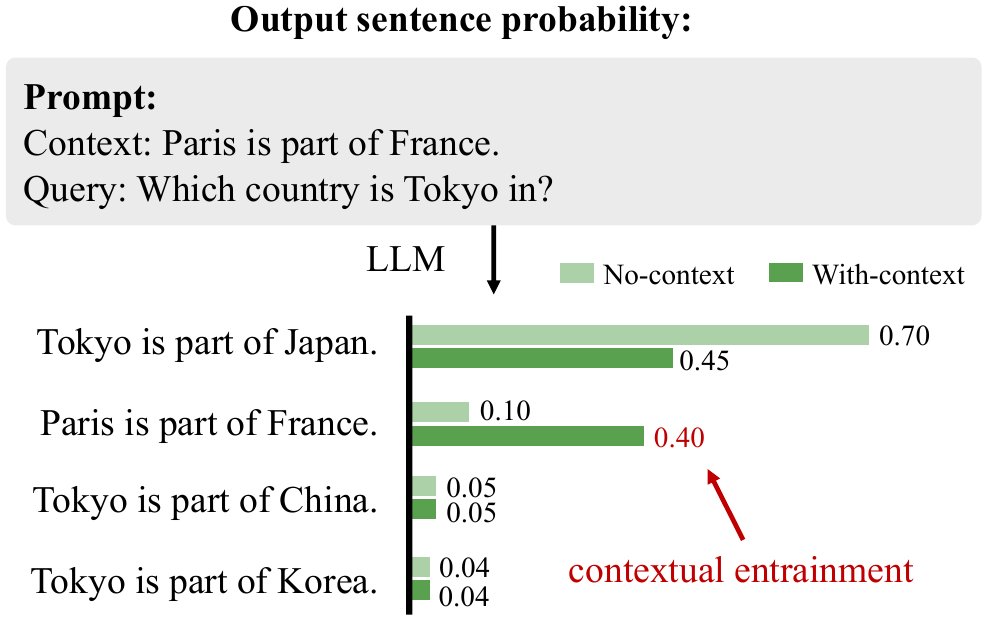}
    \caption{Sentence-level contextual entrainment (Ours): over candidate \emph{full-sentence} continuations, the in-prompt sentence (``Paris is part of France.'') receives a large probability increase.}
    \label{fig:teaser-sent}
  \end{subfigure}
  \caption{Examples of token-level and sentence-level contextual entrainment.}
  \label{fig:teaser}
\end{figure*}

%% file: tables/table_rq1_lre.tex
\begin{table*}[t]
\centering\small
\setlength{\tabcolsep}{4pt}
\begin{tabular}{l|ll|ll|ll|l}
\toprule
      & \multicolumn{2}{c}{\textbf{Gemma-2-9B}} & \multicolumn{2}{c}{\textbf{Llama-3.1-8B}} & \multicolumn{2}{c}{\textbf{Mistral-7B}} & \textbf{Avg.} \\
\cmidrule(lr){2-3} \cmidrule(lr){4-5} \cmidrule(lr){6-7}
      & \textbf{Base} & \textbf{IT} & \textbf{Base} & \textbf{IT} & \textbf{Base} & \textbf{IT} &  \\
\midrule
no-context        & $57.0$           & $80.8$           & $79.4$           & $79.2$           & $67.4$            & $80.8$           & $74.1$ \\
\midrule
factual statement & $7.6_{-49.4}$    & $74.3_{-6.6}$    & $44.7_{-34.7}$   & $73.0_{-6.2}$    & $54.4_{-12.9}$    & $81.4_{+0.6}$    & $55.9_{-18.2}$ \\
\rowcolor{red!8}
counterfactual statement & $29.3_{-27.8}$   & $69.0_{-11.8}$   & $22.4_{-57.0}$   & $55.3_{-23.9}$   & $38.1_{-29.3}$    & $64.6_{-16.2}$   & $46.5_{-27.6}$ \\
\rowcolor{blue!8}
self              & $57.6_{+0.6}$    & $53.2_{-27.6}$   & $75.4_{-4.0}$    & $80.8_{+1.6}$    & $69.0_{+1.6}$     & $82.8_{+2.1}$    & $69.8_{-4.3}$ \\
\midrule
Avg.\             & $31.5_{-25.5}$   & $65.5_{-15.3}$   & $47.5_{-31.9}$   & $69.7_{-9.5}$    & $53.8_{-13.5}$    & $76.3_{-4.5}$    & $57.4_{-16.7}$ \\
\bottomrule
\end{tabular}
\caption{
      Model accuracy on the LRE dataset (\textbf{RQ1}). 
      ``no-context'' denotes the model's natural response without context; 
      ``factual statement'' uses the factual context;
      ``counterfactual statement'' uses the counterfactual context;
      ``self'' uses the model's own ``no-context'' natural response as context. 
      \textbf{Base} denotes the base model and \textbf{IT} the instruction-tuned model.
      {\setlength{\fboxsep}{1pt}\colorbox{blue!8}{Blue shading}} marks the context condition with the smallest effect on accuracy. {\setlength{\fboxsep}{1pt}\colorbox{red!8}{Red shading}} marks the condition that reduces accuracy the most.}
\label{tab:rq1-lre}
\end{table*}

%% file: tables/table_rq1_wvs.tex
\begin{table}[t]
\centering\scriptsize
\setlength{\tabcolsep}{4pt}
\begin{tabular}{l|ll|ll|ll|l}
\toprule
      & \multicolumn{2}{c}{\textbf{Gemma-2-9B}} & \multicolumn{2}{c}{\textbf{Llama-3.1-8B}} & \multicolumn{2}{c}{\textbf{Mistral-7B}} & \textbf{Avg.} \\
\cmidrule(lr){2-3} \cmidrule(lr){4-5} \cmidrule(lr){6-7}
      & \textbf{Base} & \textbf{IT} & \textbf{Base} & \textbf{IT} & \textbf{Base} & \textbf{IT} &  \\
\midrule
$\rho^{n_s}_{c_s}$       & $83.8$ & $84.4$ & $79.8$ & $80.5$ & $80.2$ & $84.1$ & $82.1$ \\
\rowcolor{gray!20}
$\rho^{n_s}_{c_o}$       & $73.5$ & $82.1$ & $72.2$ & $67.8$ & $67.5$ & $77.0$ & $73.4$ \\
\rowcolor{yellow!20}
$\rho^{n_s}_{\text{self}}$ & $99.0$ & $81.8$ & $99.0$ & $93.9$ & $98.9$ & $93.5$ & $94.4$ \\
\midrule
\rowcolor{gray!20}
$\rho^{n_o}_{c_s}$       & $52.9$ & $58.1$ & $50.6$ & $54.2$ & $47.2$ & $66.3$ & $54.9$ \\
$\rho^{n_o}_{c_o}$       & $63.9$ & $66.8$ & $55.6$ & $66.7$ & $59.5$ & $68.9$ & $63.6$ \\
\rowcolor{yellow!20}
$\rho^{n_o}_{\text{self}}$ & $87.1$ & $52.8$ & $81.2$ & $79.2$ & $82.1$ & $92.1$ & $79.1$ \\
\bottomrule
\end{tabular}
\caption{
      Model opinion consistency on the WVS Dataset (\textbf{RQ1}).
      $\rho_{y}^{x}$ denotes the consistency of opinion polarity between the model's natural response (of polarity $x$) given context $y$ and its natural response without context; 
      $c_s$ and $c_o$ denote contexts that support and oppose the statement, 
      while $n_s$ and $n_o$ denote whether the model's natural response is closer to the supporting or the opposing statement relative to query $q$.  
      ``self'' uses the model's own ``no-context'' natural response as context.
      Each cell is the percentage of queries whose with-context opinion still matches the no-context opinion.
      {\setlength{\fboxsep}{1pt}\colorbox{yellow!20}{Yellow shading}} marks the context condition with the smallest effect on opinion consistency. {\setlength{\fboxsep}{1pt}\colorbox{gray!20}{Gray shading}} marks the condition with the largest effect (the context whose polarity opposes the model's natural response).}
\label{tab:rq1-wvs}
\end{table}

%% file: figures/fig_rq1_lre_bars.tex
\begin{figure}[t]
  \centering
  \includegraphics[width=0.95\linewidth]{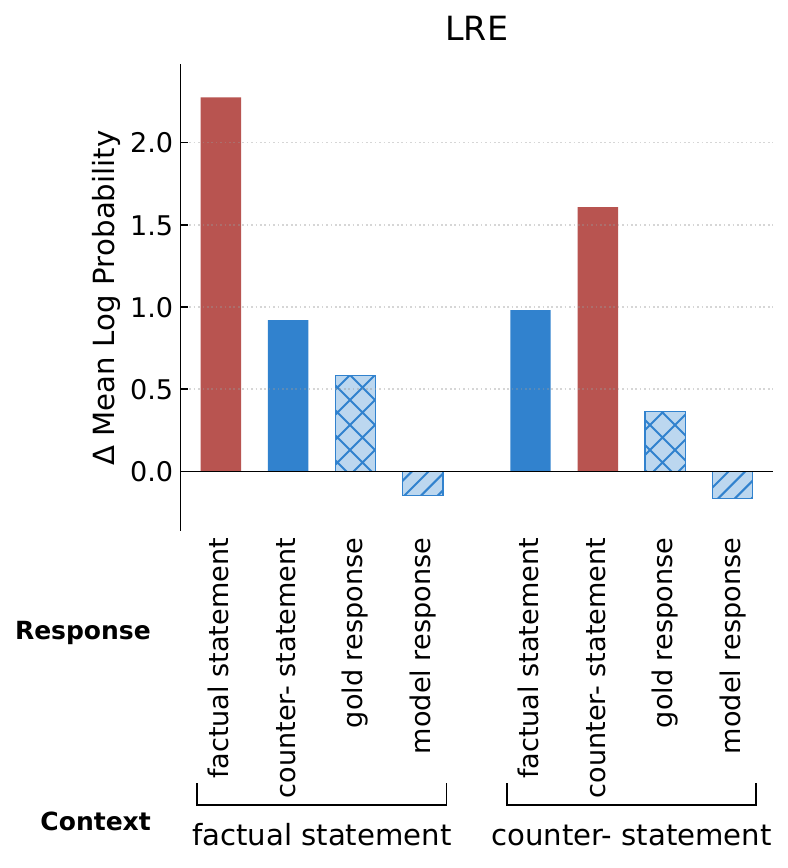}
  \caption{
    Sentence-level contextual entrainment on Llama-3.1-8B (\textbf{RQ1}).
    \textbf{Context} indicates the context type, which supports two settings: factual and counterfactual (counter-);
    \textbf{Response} indicates the text on which the log probability is computed.
    The \textcolor[HTML]{B85450}{\textbf{red}} bars represent contextual entrainment, and the \textcolor[HTML]{3182CE}{\textbf{blue}} bars represent distraction.}
  \label{fig:rq1-lre-bars}
\end{figure}

%% file: figures/fig_rq2_wvs_panel.tex
\begin{figure*}[t]
  \centering
  \begin{subfigure}[t]{0.24\linewidth}
    \centering
    \includegraphics[width=\linewidth]{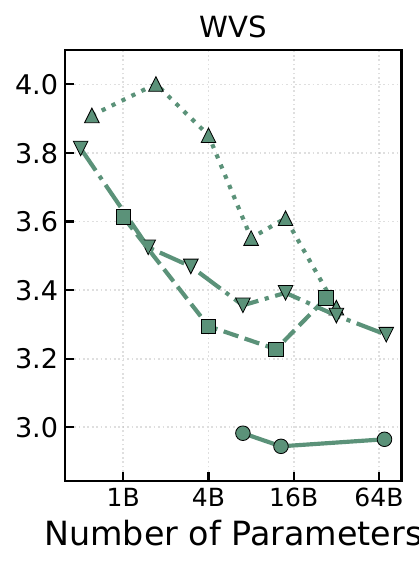}
    \caption{Support statement.}
    \label{fig:rq2-wvs-E}
  \end{subfigure}\hfill
  \begin{subfigure}[t]{0.24\linewidth}
    \centering
    \includegraphics[width=\linewidth]{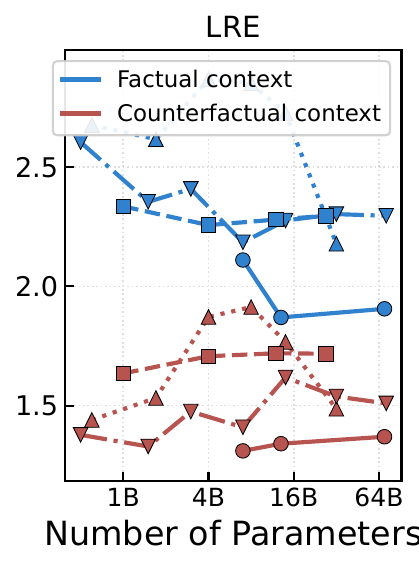}
    \caption{Counter- and Factual.}
    \label{fig:rq2-lre-E}
  \end{subfigure}\hfill
  \begin{subfigure}[t]{0.24\linewidth}
    \centering
    \includegraphics[width=\linewidth]{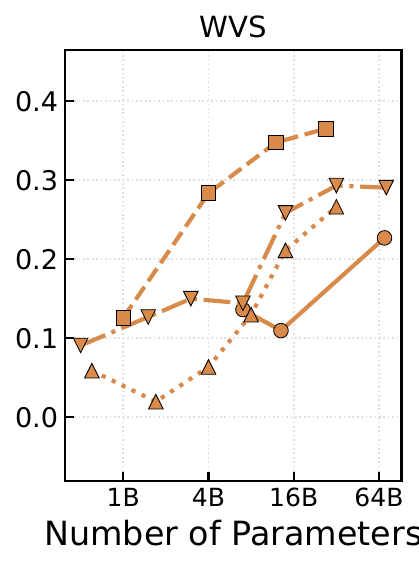}
    \caption{Support response.}
    \label{fig:rq2-wvs-I}
  \end{subfigure}\hfill
  \begin{subfigure}[t]{0.24\linewidth}
    \centering
    \includegraphics[width=\linewidth]{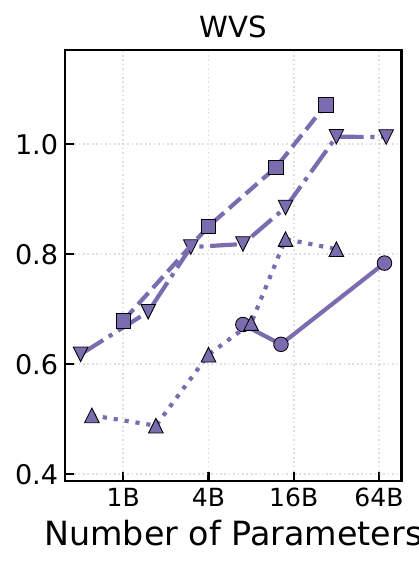}
    \caption{Oppose statement.}
    \label{fig:rq2-wvs-O}
  \end{subfigure}

  \caption{Model-size effects on sentence-level contextual entrainment (\textbf{RQ2}).
  The context for (a), (c), and (d) is \textbf{support statement};
    the caption represents the response.
    (b) is that the response equals the context, under factual and counterfactual contexts.
  Model family encoding: \llamamarker{}~Llama-2,  \qwenthreemarker{}~Qwen3,  \qwentwofivemarker{}~Qwen2.5,  \gemmamarker{}~Gemma-3.
  The y-axis shows the ``$\Delta$ Mean Log Probability'' under a specific context and response setting.
  }
  \label{fig:rq2-wvs}
\end{figure*}

%% file: tables/table_rq3_entrainment.tex
\begin{table*}[t]
\centering
\footnotesize
\setlength{\tabcolsep}{4pt}
\begin{tabular}{l lllllll}
\toprule
  & \multicolumn{2}{c}{\textbf{Gemma-2-9B}} & \multicolumn{2}{c}{\textbf{Llama-3.1-8B}} & \multicolumn{2}{c}{\textbf{Mistral-7B}} & \textbf{Avg.} \\
\cmidrule(lr){2-3} \cmidrule(lr){4-5} \cmidrule(lr){6-7}
  & \textbf{Base} & \textbf{IT} & \textbf{Base} & \textbf{IT} & \textbf{Base} & \textbf{IT} &   \\
\midrule
Unmasked & $+2.42$ & $+2.32$ & $+2.11$ & $+2.14$ & $+1.67$ & $+1.48$ & $+2.02$ \\
Random head & $+2.41_{-0.01}$ & $+2.31_{-0.01}$ & $+2.11_{-0.00}$ & $+2.14_{+0.01}$ & $+1.65_{-0.03}$ & $+1.47_{-0.01}$ & $+2.01_{-0.01}$ \\
\rowcolor{green!12}
Per-relation head & $+0.56_{-1.86}$ & $+0.69_{-1.63}$ & $-0.31_{-2.42}$ & $-0.02_{-2.16}$ & $-0.14_{-1.81}$ & $-0.48_{-1.96}$ & $+0.05_{-1.97}$ \\
\rowcolor{orange!15}
Shared head & $+1.26_{-1.16}$ & $+1.24_{-1.09}$ & $+1.46_{-0.65}$ & $+1.06_{-1.08}$ & $+0.70_{-0.97}$ & $+0.93_{-0.55}$ & $+1.11_{-0.91}$ \\
\bottomrule
\end{tabular}
\caption{Contextual entrainment averaged over 16 LRE relations (\textbf{RQ3}).
``Unmasked'' is the original model without any attention heads masked.
``Random head'' is averaged over 10 seeds of random head sets matched in size to the set of shared heads.
``Per-relation head'' uses each relation's own trained mask;
``Shared head'' uses the shared head mask (2\%-4\% of attention heads).
{\setlength{\fboxsep}{1pt}\colorbox{green!12}{Green shading}} marks the per-relation mask, which almost entirely eliminates contextual entrainment; {\setlength{\fboxsep}{1pt}\colorbox{orange!15}{Orange shading}} marks the shared head mask, which roughly halves entrainment while preserving model performance.}
\label{tab:rq3-entrainment}
\end{table*}

%% file: tables/table_rq3_accuracy.tex
\begin{table*}[t]
\centering
\footnotesize
\setlength{\tabcolsep}{4pt}
\begin{tabular}{l l lllllll}
\toprule
  &   & \multicolumn{2}{c}{\textbf{Gemma-2-9B}} & \multicolumn{2}{c}{\textbf{Llama-3.1-8B}} & \multicolumn{2}{c}{\textbf{Mistral-7B}} & \textbf{Avg.} \\
\cmidrule(lr){3-4} \cmidrule(lr){5-6} \cmidrule(lr){7-8}
\textbf{Prompt} & \textbf{Mask} & \textbf{Base} & \textbf{IT} & \textbf{Base} & \textbf{IT} & \textbf{Base} & \textbf{IT} &   \\
\midrule
\multirow{5}{*}{w/o context} & Unmasked & $61.5$ & $84.8$ & $80.1$ & $82.8$ & $71.0$ & $84.7$ & $77.5$ \\
 \cmidrule(lr){2-9}
 & Random head & $\textbf{56.9}_{-4.6}$ & $84.8_{-0.0}$ & $73.1_{-6.9}$ & $\textbf{82.8}_{-0.0}$ & $59.0_{-11.9}$ & $83.7_{-1.0}$ & $73.4_{-4.1}$ \\
 & Per-relation head & $39.9_{-21.6}$ & $82.6_{-2.2}$ & $74.8_{-5.3}$ & $78.4_{-4.4}$ & $66.0_{-5.0}$ & $82.7_{-2.0}$ & $70.7_{-6.8}$ \\
 & Shared head & $52.8_{-8.7}$ & $\textbf{85.6}_{+0.8}$ & $\textbf{77.4}_{-2.7}$ & $82.2_{-0.7}$ & $\textbf{68.3}_{-2.7}$ & $\textbf{84.7}_{-0.1}$ & $\textbf{75.2}_{-2.3}$ \\
\midrule
\multirow{5}{*}{w/ context} & Unmasked & $7.8$ & $75.6$ & $60.5$ & $77.7$ & $56.9$ & $85.4$ & $60.7$ \\
\cmidrule(lr){2-9}
 & Random head& $\textbf{10.4}_{+2.6}$ & $74.5_{-1.1}$ & $56.9_{-3.6}$ & $75.5_{-2.2}$ & $51.6_{-5.3}$ & $84.0_{-1.4}$ & $58.8_{-1.8}$ \\
 & Per-relation head & $9.5_{+1.7}$ & $\textbf{82.2}_{+6.6}$ & $62.2_{+1.6}$ & $78.7_{+1.0}$ & $58.3_{+1.3}$ & $81.5_{-3.9}$ & $62.1_{+1.4}$ \\
 & Shared head & $9.3_{+1.5}$ & $81.0_{+5.4}$ & $\textbf{64.9}_{+4.4}$ & $\textbf{82.2}_{+4.5}$ & $\textbf{64.9}_{+7.9}$ & $\textbf{84.4}_{-1.0}$ & $\textbf{64.5}_{+3.8}$ \\
\bottomrule
\end{tabular}
\caption{Free-generation accuracy on the 16 LRE relations (\textbf{RQ3}).
``w/o context'' is the bare query (capability preservation); 
``w/ context'' is the query preceded by the factual statement (entrainment robustness). 
Subscripts give the change in accuracy relative to the ``Unmasked'' row.
\textbf{Bold} marks the highest free-generation accuracy among the three masking strategies.}
\label{tab:rq3-accuracy}
\end{table*}